\documentclass{article}

\usepackage{microtype}
\usepackage{graphicx}
\usepackage{booktabs} 

\usepackage{hyperref}
\usepackage{bm}
\usepackage{mathtools}
\usepackage{subcaption}
\usepackage[utf8]{inputenc}
\usepackage[english]{babel}
\usepackage{algorithmicx}
\usepackage[noend]{algpseudocode}
\usepackage{graphicx}
\usepackage{amsmath,amssymb,amsfonts, amsthm}
\newtheorem{theorem}{Theorem}
\newtheorem{lemma}{Lemma}
\newtheorem{corollary}{Corollary}
\usepackage{bbm}
\usepackage{float}
\usepackage{enumitem,kantlipsum}

\DeclarePairedDelimiter\abs{\lvert}{\rvert}

\usepackage[accepted]{icml2019}

\graphicspath{{newplots/}}

\begin{document}

\twocolumn[
\icmltitle{Almost Boltzmann Exploration}




\begin{icmlauthorlist}
\icmlauthor{Harsh Gupta}{to,goo}
\icmlauthor{Seo Taek Kong}{to,goo}
\icmlauthor{R. Srikant}{to,goo}
\icmlauthor{Weina Wang}{ed}
\end{icmlauthorlist}
\icmlcorrespondingauthor{Harsh Gupta}{hgupta10@illinois.edu}
\icmlaffiliation{to}{Coordinated Science Laboratory, University of Illinois at Urbana-Champaign}
\icmlaffiliation{goo}{Department of Electrical and Computer Engineering, University of Illinois at Urbana-Champaign}

\icmlaffiliation{ed}{Department of Computer Science, Carnegie Mellon}


\icmlkeywords{Machine Learning, ICML}

\vskip 0.3in
]



\printAffiliationsAndNotice{}  

\begin{abstract}
Boltzmann exploration is widely used in reinforcement learning to provide a trade-off between exploration and exploitation. Recently, in \cite{Gumbel} it has been shown that pure Boltzmann exploration does not perform well from a regret perspective, even in the simplest setting of stochastic multi-armed bandit (MAB) problems. In this paper, we show that a simple modification to Boltzmann exploration, motivated by a variation of the standard doubling trick, achieves $O(K\log^{1+\alpha} T)$ regret for a stochastic MAB problem with $K$ arms, where $\alpha>0$ is a parameter of the algorithm. This improves on the result in \cite{Gumbel}, where an algorithm inspired by the Gumbel-softmax trick achieves $O(K\log^2 T)$ regret. We also show that our algorithm achieves $O(\beta(G) \log^{1+\alpha} T)$ regret in stochastic MAB problems with graph-structured feedback, without knowledge of the graph structure, where $\beta(G)$ is the independence number of the feedback graph. Additionally, we present extensive experimental results on real datasets and applications for multi-armed bandits with both traditional bandit feedback and graph-structured feedback. In all cases, our algorithm performs as well or better than the state-of-the-art.
\end{abstract}

\section{Introduction}
Boltzmann exploration is a widely used technique in reinforcement learning to provide a trade-off between exploration and exploitation of various available actions (see \cite{bertsekas1995neuro} and \cite{sutton2018reinforcement} for details). However, even in the simplest reinforcement learning setting of stochastic multi-armed bandits, it has not been studied extensively. Recently, \cite{Gumbel} analyzed the following variant of Boltzmann exploration for the stochastic multi-armed bandit problem:
\begin{align}\label{eq:BE}
    p_{t, k} \propto e^{\eta_t\hat{\mu}_{t, k}}
\end{align}
where $p_{t, k}$ is the probability of pulling arm $k$ at time $t$, $\eta_t > 0$ is the learning rate and $\hat{\mu}_{t, k}$ is the average reward obtained for arm $k$ until time $t$. They showed that any monotone learning rate sequence for the above variant of Boltzmann exploration will be sub-optimal. Further, they utilized the Gumbel-softmax trick to propose an algorithm called Boltzmann-Gumbel Exploration (BGE), which was shown to achieve $O(K\log^2T)$ regret. 

It is also worth noting that multi-armed bandits have been studied using two distinct models for reward generation: stochastic and adversarial (see \cite{bubeck2012regret} for a survey on bandits). In the stochastic case, the rewards for each arm are drawn i.i.d. from a probability distribution. On the other hand, in the adversarial case, the rewards can be arbitrarily chosen by an adversary. A number of algorithms have been proposed for both these models, including UCB, KL-UCB, and Thompson sampling for the stochastic case, and EXP3 and INF for the adversarial case (see \cite{bubeck2012regret} and  \cite{cesa2006prediction} for more details). Algorithms that perform well for both models have been studied less extensively. Recently, \cite{seldin2014one} designed an algorithm that achieves order-optimal performance for the adversarial case, but achieves sub-optimal $O(\log^3 T)$ regret for the stochastic case. Therefore, a natural question to ask is which algorithm should one use in practice?

The above question was explored in \cite{kegl2009boosting} and \cite{Fekete10} in the context of application of bandits to increase the efficiency (speed) of boosting algorithms. In \cite{kegl2009boosting}, the authors show that using UCB leads to significant improvements in speed over traditional boosting. In \cite{Fekete10}, the authors show that using a variant of EXP3 (known as EXP3.P) further improves performance in some cases and performs as well as using UCB in most cases. In this paper, we will also consider the same experimental setting (in addition to other experiments) considered in the above two papers to test the efficacy of our algorithm.

It is worth noting that in the traditional multi-armed bandit problem, the player only observes the reward of the arm that he/she pulled. This scenario is known as bandit feedback. A generalization of this feedback mechanism is the graph-structured feedback mechanism (see \cite{mannor2011bandits}, \cite{alon2017nonstochastic} and \cite{liu2018analysis}). To the best of our knowledge, Boltzmann exploration has not been studied in the more general graph-structured feedback context. 

Motivated by these considerations, we aim to answer the following questions in our paper:
\begin{enumerate}[leftmargin=*, noitemsep, topsep=0pt]
    \item Is there a simple modification to Boltzmann exploration that achieves at least the $O(K\log^2T)$ regret achieved by the Boltzmann-Gumbel exploration algorithm presented in \cite{Gumbel}?
    \item If such a modification exists, how well does it perform in practice, especially on the real datasets and the application considered in \cite{Fekete10}?
    \item Does the modified Boltzmann exploration algorithm extend naturally to other bandit problems, such as those with graph-structured feedback? If so, then how does it perform in practice?
\end{enumerate}
We answer these questions and more in this paper. Specifically, our contributions are as follows:
\begin{enumerate}[leftmargin=*, noitemsep, topsep=0pt]
    \item We show that adding pure exploration phases (the lengths of which increase with time) in between Boltzmann exploration phases achieves $O(K\log^{1 + \alpha}T)$ regret for the stochastic multi-armed bandit problem, where $\alpha > 0$ is an algorithmic hyperparameter. Note that the big-O notation hides all problem and algorithm dependent parameters (except $K$) which do not affect the dependence on time.
    \item We achieve the above regret with a simple modification to the doubling trick. This modification allows arbitrarily small amounts of pure exploration, and for this reason, we call our algorithm Almost Boltzmann Exploration or ABE. We note that, Boltzmann exploration can be recovered as a limiting case of ABE, but the regret bounds require at least small amounts of pure exploration.
    \item We also show the above algorithm naturally extends to bandit problems with graph-structured feedback, achieving $O(\beta(G)\log^{1+\alpha} T)$ regret, where $\beta(G)$ is the independence number of the feedback graph $G$. It is worth noting that ABE does not require the graph knowledge.
    \item To test our algorithm, as in \cite{Fekete10}, we conduct experiments on real datasets (MNIST, USPS, UCI pendigit, UCI letter and UCI ISOLET) where bandit algorithms are used to speed up the implementation of AdaBoost. We first applied ABE to these datasets, with different choices of hyperparameters to control the amount of pure exploration, and found that very small amounts of pure exploration work well in practice. With these parameters, ABE further performs as well as or better than the best algorithm (EXP3.P) in \cite{Fekete10}. 
    \item We also conducted experiments using the Epinions dataset and Facebook graphs to model graph-structured feedback problems, and show that ABE outperforms the state-of-the-art. The same remark as in the previous bullet applies to the choice of hyperparameters in the graph-structured feedback setting as well.
\end{enumerate}

\section{Algorithms}
\subsection{Stochastic MAB Problem}
In this paper, we consider the stochastic $K$-armed bandit problem with bandit and graph-structured feedback. Here, we briefly introduce the problem with bandit feedback.  

The traditional stochastic MAB problem can be described as follows: at each time slot $t$, there are $K$ arms available for a player to pull. After pulling an arm, the player receives a Bernoulli reward, which is assumed to be i.i.d across time slots for each arm. For each arm $k$, let the expected reward be denoted by $\mu_k$. Without loss of generality, we assume that $\mu_1 > \mu_k, \forall k \neq 1$. Also, let $\Delta_k = \mu_1 - \mu_k$ denote the difference in the expected reward of arm $k$ from the optimal arm. Let the time horizon be denoted by $T$. The objective of the problem is to design an algorithm for the player such that at the end of time $T$, the cumulative reward obtained by the player is as close to the cumulative reward obtained by pulling the optimal arm at each time slot. A popular measure to quantify the performance of an algorithm for the stochastic multi-armed problem is the expected regret, which is defined as follows:
\begin{align*}
    \mathbb{E}[R(T)] = \sum_{k = 2}^K \Delta_k\mathbb{E}[\tau_{k}(T)]
\end{align*}
where $\tau_k(t)$ denotes the number of times arm $k$ has been pulled until the end of time $t$. 

As shown in \cite{Gumbel}, the reason that Boltzmann exploration does not work for the stochastic multi-armed bandit problem is the following: \textit{there is a chance that in the initial phase a sub-optimal arm receives the maximum average reward, leading to a scenario in which there will be at least a constant probability of that sub-optimal arm being sampled for the remaining amount of time}. Essentially, the primary issue with Boltzmann exploration is that it might not explore all the arms enough number of times (in order to have good estimates of all the expected rewards), especially when a sub-optimal arm receives a higher average reward than the optimal arm in the initial phase of the algorithm. In order to fix this issue, we propose the Almost Boltzmann Exploration (ABE) algorithm.

The basic idea behind ABE is the following: \textit{in order to ensure that all the arms are sufficiently explored, ABE intersperses small periods of pure exploration between longer periods of Boltzmann exploration}. The interspersed periods of pure exploration allow ABE to explore all arms enough number of times in the cases where Boltzmann exploration fails, while at the same time retaining the advantage of using Boltzmann exploration in other cases. More precisely, we split the algorithm into two inner loops, controlled by an outer loop. In each outer loop $i$, we have the following:
\begin{itemize}[leftmargin=*, noitemsep, topsep=0pt]
    \item The first inner loop, which we refer to as the pure exploration phase, consists of exploring each arm $\lfloor ci^{\alpha}\rfloor$ number of times. Here, $c > 0$ and $\alpha >0$ are the algorithm parameters that control the amount of pure exploration. 
    \item The second inner loop, which we refer to as the Boltzmann exploration phase, implements Boltzmann exploration (as in \eqref{eq:BE}), with $\eta_t \propto \sqrt{t}$. In this phase, for outer iteration $i$, there are a total of $K\times 2^i$ time slots, which is much larger than the $K\times \lfloor ci^{\alpha}\rfloor$ time slots in the pure exploration phase. This allows the overall algorithm to be quite close to Boltzmann exploration, thus justifying its name.
\end{itemize}

The Almost Boltzmann Exploration (ABE) algorithm for the standard stochastic multi-armed bandit problem is formally presented as Algorithm \ref{alg:abe}.

\begin{algorithm}[H]
\caption{Almost Boltzmann Exploration (ABE)}\label{alg:abe}
\begin{algorithmic}[1]\\
\textbf{parameters}: $\eta , \alpha, c > 0$.\\
\textbf{initialize}: $r_i(0) = 0$,  $\tau_i(0) = 0 $, $\forall i \in [K]$, and $t = 0$.
\For{$i = 1, 2, ...,$}
\For{$j = 1, 2, 3, ..., K\lfloor ci^{\alpha}\rfloor$}
\State $t = t + 1$.
\State Play arm $k_t = j \mod K$.
\State Receive reward $x_{k_t}(t)$.
\State $\tau_{k_t}(t) = \tau_{k_t}(t - 1) + 1$.
\State $r_{k_t}(t) = \frac{(\tau_{k_t}(t) - 1)r_{k_t}(t - 1) + x_{k_t}(t)}{\tau_{k_t}(t)}$.
\EndFor
\For{$j = 1, 2, 3, ..., K\times2^{i}$}
\State $t = t + 1$.
\State Compute $p_k(t) \propto \exp\big(\eta \sqrt{(t-1)}r_k\big)$, $\forall k$.
\State Play arm $k_t \sim p_k(t)$.
\State Receive reward $x_{k_t}(t)$.
\State $\tau_{k_t}(t) = \tau_{k_t}(t - 1) + 1$.
\State $r_{k_t}(t) = \frac{(\tau_{k_t}(t) - 1)r_{k_t}(t - 1) + x_{k_t}(t)}{\tau_{k_t}(t)}$.
\EndFor
\EndFor\label{WISHForLoop}
\end{algorithmic}
\end{algorithm}

\subsection{Graph-Structured Feedback}
In the stochastic multi-armed bandit problem with graph-structured feedback, in addition to receiving the reward for the pulled arm, one also observes rewards from other arms according to an underlying feedback graph $G = (V, E).$ In this graph, each vertex $v\in V$ represents an arm and an edge $(u,v)\in E$ between arms $u$ and $v$ indicates that when any of these arms is pulled, we also observe the reward for the other arm. This additional feedback can be used to generate better estimates of expected rewards for each arm with fewer pulls. In this paper, we will focus on time-invariant undirected feedback graphs; the extension to the time-varying case is straightforward.

ABE (Algorithm \ref{alg:abe}) can be directly applied to the above problem scenario, but it does not exploit the graph-structured feedback and hence will be sub-optimal. The reason is that during the pure exploration phase of ABE, it pulls each arm sequentially to receive information regarding the rewards of all arms. This is not necessary in the graph-structured case as pulling one arm gives information on the rewards of its neighbouring arms in the feedback graph $G$. In fact, the number of pulls needed to obtain information for all arms in the pure exploration phase can be reduced from $K$ to at most $\beta(G)$, where $\beta(G)$ is the independence number of the feedback graph $G$. The idea behind doing that is as follows:
\begin{itemize}[leftmargin=*, noitemsep, topsep=0pt]

\item Within each pure exploration phase, we can get feedback from all the arms without pulling all of them in the following manner: pull a randomly selected arm, get its reward, and observe the rewards of all its neighbors. Remove this arm and its neighbors from the set of arms, randomly select one of the remaining arms and repeat the process till no more arms are left.
\item  We note that the above procedure to pull arms \emph{does not require the knowledge of the graph structure} since we only need to know the arms for which the rewards are observed. Further, it is easy to see that the above procedure will pull a number of arms less than equal to $\beta(G)$. 

\item For efficient implementation, one does not have to repeatedly select arms in a random fashion as mentioned above. Instead, one can perform this procedure once and then pull the arms in the same sequence in all future pure exploration phases. For ease of exposition, we assume that the arms are pulled in a specific sequence $\xi$ (of length less than or equal to $\beta(G)$) during each pure exploration phase. 
\end{itemize}

Since the algorithm for the graph-structured feedback problem is similar to Algorithm \ref{alg:abe}, except for the above-mentioned modification to the pure exploration phase, the pseudocode for this algorithm is presented in the supplementary material, and will be referred to as Algorithm 2 in the rest of the main body of the paper. The algorithm can also be used for directed graphs as we will demonstrate in the experimental section. But we do not theoretically analyze this case since we believe that the performance of the algorithm can be improved by taking into account the directed nature of the graph; we do not pursue this direction of research here, but it is a topic of ongoing research.

\section{Analysis}
In this section, we analyze Algorithm \ref{alg:abe} and Algorithm 2. All proofs that are not in the main body of the paper can be found in the supplementary material. We begin with Algorithm \ref{alg:abe}.

\subsection{Analysis of Algorithm \ref{alg:abe}}
We begin by presenting our first lemma showing that the outer iteration number for any time $t$ is of logarithmic order in $t$.
\begin{lemma}\label{lemma1}
For any time $t > 0$, the corresponding outer iteration number $i_t$ satisfies:
\begin{align*}
    \log (\frac{t}{2(1 + 2c)K}) \leq i_t \leq \log (\frac{t}{K} + 2)
\end{align*}
\hfill $\qedsymbol$
\end{lemma}

Recall that $\tau_k(t)$ denotes the number of times arm $k$ has been pulled until the end of time $t$. Next, we show that the pure exploration phase of Algorithm \ref{alg:abe} (i.e., the first inner loop of ABE) ensures that each arm is explored at least $O(\log^{1 + \alpha}(t))$ times until time $t$, for a sufficiently large $t$.
\begin{lemma}\label{lemma2}
The number of times any arm $k$ has been played until the end of time $t$, for any $t$ such that $\log(\frac{t}{2(1 + 2c)K}) \geq \{\frac{2(\alpha + 1)}{c}\}^{\frac{1}{\alpha}}$ and $t$ falls in the Boltzmann exploration phase of ABE (Algorithm \ref{alg:abe}), is lower bounded by:
\begin{align*}
    \tau_k(t) \geq \frac{c}{2}\frac{\{\log(\frac{t}{2(1 + 2c)K})\}^{\alpha + 1}}{\alpha + 1}, \forall k \in [K].
\end{align*}
\end{lemma}
\begin{proof}
From Lemma \ref{lemma1}, we know that the outer iteration number $i_t$ for time $t$ is lower bounded by $\log(\frac{t}{2(1 + 2c)K})$. 
Therefore, the pure exploration phase, i.e., the first inner loop ensures that each arm $k$ has been pulled at least the following number of times:
\begin{align}\label{eq:lemma2eq}
    \begin{split}
    \tau_k(t) \geq \sum_{k = 1}^{i_t}\lfloor ck^{\alpha}\rfloor &\geq \sum_{k = 1}^{i_t} (ck^{\alpha} - 1)\\
    &\geq c\int_{0}^{i_t}k^{\alpha}dk - i_t\\
    &\geq c \frac{i_t^{\alpha + 1}}{\alpha + 1} - i_t
    \end{split}
\end{align}
Consider the function $f(x) = c\frac{x^{\alpha + 1}}{\alpha + 1} - x$ from the RHS of the last inequality. We observe that $f(x) \geq \frac{c}{2}\frac{x^{\alpha + 1}}{\alpha + 1}, \forall x \geq \{\frac{2(\alpha + 1)}{c}\}^{\frac{1}{\alpha}}$. Combining this fact and \eqref{eq:lemma2eq}, along with the fact that $i_t \geq \log(\frac{t}{2(1 + 2c)K}) \geq \{\frac{2(\alpha + 1)}{c}\}^{\frac{1}{\alpha}}$, we get the lemma.
\end{proof}
Next, we upper bound the number of times each arm will be pulled in the pure exploration phase of Algorithm \ref{alg:abe}, until the end of time horizon $T$.
\begin{lemma}\label{lemma3}
Let $\tau'_k(T)$ denote the number of times arm $k$ is pulled during the pure exploration phase of ABE (Algorithm \ref{alg:abe}), until the end of time horizon $T$. Then:
\begin{align*}
    \tau'_k(T) \leq c\frac{\{\log(\frac{T}{K} + 2) + 1\}^{\alpha + 1}}{\alpha + 1}, \forall k \in [K].
\end{align*}
\end{lemma}
\begin{proof}
From Lemma \ref{lemma1}, we know that the outer iteration number $i_T$ for time $T$ is upper bounded by $\log(\frac{T}{K} + 2)$. Therefore, during the pure exploration phase, each arm $k$ will be pulled at most the following number of times:
\begin{align*}
    \begin{split}
    \tau'_k(T) \leq \sum_{k = 1}^{i_T}\lfloor ck^{\alpha}\rfloor &\leq \sum_{k = 1}^{i_T} ck^{\alpha}\\
    &\leq c\int_{0}^{i_T + 1}k^{\alpha}dk\\
    &\leq c \frac{(i_T + 1)^{\alpha + 1}}{\alpha + 1}\\
    &\leq c\frac{\{\log(\frac{T}{K} + 2) + 1\}^{\alpha + 1}}{\alpha + 1}.
    \end{split}
\end{align*}
\end{proof}
Next, we show that the probability of sampling a sub-optimal arm in the Boltzmann exploration phase of Algorithm \ref{alg:abe} (i.e., the second inner loop of ABE) is small, after a sufficiently large time $t$. 
\begin{lemma}\label{lemma4}
For any time $t$ such that $\log(\frac{t}{2(1 + 2c)K}) \geq \gamma_k = \{\frac{16\ln 2}{\Delta_k^2}\frac{2(\alpha + 1)}{c}\}^{\frac{1}{\alpha}}$ and $t$ falls in the Boltzmann exploration phase of ABE (Algorithm \ref{alg:abe}), the following holds:
\begin{align*}
    p_k(t) \leq e^{-\frac{\eta\Delta_k}{2}\sqrt{t}} + \frac{128(1 + 2c)^2K^2}{(\Delta_kt)^2},
\end{align*}
where $p_k(t)$ is the probability of sampling a sub-optimal arm $k$ at time $t$ as defined in the Boltzmann exploration phase of ABE (or the second inner loop of ABE). 
\end{lemma}
\begin{proof}
Let $\hat{\mu}_k(t - 1) = r_k(t - 1)$, denote the average reward for arm $k$, at the end of time $t - 1$ or the beginning of time $t$. Consider the following events:
\begin{itemize}
    \item $A^k_t : \hat{\mu}_k(t - 1) - \hat{\mu}_1(t - 1) < -\frac{\Delta_k}{2}$.
    \item $B^k_t : \hat{\mu}_k(t - 1) \geq \mu_k + \frac{\Delta_k}{4}$.
    \item $C^k_t : \hat{\mu}_1(t - 1) \leq \mu_1 - \frac{\Delta_k}{4}$.
\end{itemize}
Note that, for any $k$, if both the events $B^k_t$ and $C^k_t$ occur, then the event $A^k_t$ will not occur. Further, it can be easily verified that $A^k_t\cup B^k_t\cup C_k^t$ is the entire sample space. Therefore,
\begin{align}\label{eq:lem2}
\begin{split}
    \mathbbm{1}&\{k_t = k\} \\&\leq \mathbbm{1}\{k_t = k, A^k_t\} + \mathbbm{1}\{k_t = k, B^k_t\} + \mathbbm{1}\{k_t = k, C^k_t\}\\
    &\leq \mathbbm{1}\{k_t = k, A^k_t\} + \mathbbm{1}\{B^k_t\} + \mathbbm{1}\{C^k_t\}
\end{split}
\end{align}
Let us consider the first term on the RHS of the above inequality:
\begin{align}\label{eq:lem2a}
    \begin{split}
        \mathbb{E}[\mathbbm{1}\{k_t = k, A^k_t\}] &= \mathbb{P}(k_t = k, A^k_t)\\
        &\leq \mathbb{P}(k_t = k | A^k_t)\\
        &= \frac{e^{\eta\sqrt{t - 1}\hat{\mu}_k(t - 1)}}{\sum_{l = 1}^K e^{\eta\sqrt{t - 1}\hat{\mu}_l(t - 1)}} (\text{under } A^k_t)\\
        &\leq e^{\eta\sqrt{t - 1}(\hat{\mu}_k(t - 1) - \hat{\mu}_1(t - 1))} (\text{under } A^k_t)\\
        &\leq e^{-\frac{\eta\Delta_k\sqrt{t - 1}}{2}}.
    \end{split}
\end{align}
Now, observe that $0 < \Delta_k^2 \leq 1, \forall k \neq 1$. Therefore, $\log(\frac{t}{2(1 + 2c)K}) \geq \{\frac{16\ln 2}{ \Delta_k^2}\frac{2(\alpha + 1)}{c}\}^{\frac{1}{\alpha}}$ implies that $\log(\frac{t}{2(1 + 2c)K}) \geq \{\frac{16\ln 2}{ \Delta_k^2}\frac{2(\alpha + 1)}{c}\}^{\frac{1}{\alpha}} \geq \{\frac{2(\alpha + 1)}{c}\}^{\frac{1}{\alpha}}$. Hence, under the conditions on $t$ in this lemma, Lemma \ref{lemma2} also holds. From Lemma \ref{lemma2}, we know that for any sub-optimal arm $k$, $T_k(t) = \tau_k(t - 1) \geq \frac{c}{2}\frac{\{\log(\frac{t}{2(1 + 2c)K})\}^{\alpha + 1}}{\alpha + 1}$. Let $\zeta(t) = \frac{c}{2}\frac{\{\log(\frac{t}{2(1 + 2c)K})\}^{\alpha + 1}}{\alpha + 1}$. Now, applying Hoeffding's inequality to the second term on the RHS of \eqref{eq:lem2}, we get:
\begin{align}\label{eq:lem2b}
    \begin{split}
        \mathbb{E}&[\mathbbm{1}\{B_t^k\}]\\ &= \mathbb{P}(B_t^k)\\
        & = \mathbb{P}(\hat{\mu}_k(t - 1) \geq \mu_k + \frac{\Delta_k}{4})\\
        & = \mathbb{P}(\exists s \in \{\lceil\zeta(t)\rceil, ..., t - 1\} \text{ with } \tau_k(t - 1) = s, B_t^k)\\
        &\leq \sum_{s = \lceil\zeta(t)\rceil}^{t - 1}\mathbb{P}(\hat{\mu}_k(t - 1) \geq \mu_k + \frac{\Delta_k}{4}, \tau_k(t - 1) = s)\\
        &\leq \sum_{s = \lceil\zeta(t)\rceil}^{t - 1}e^{-\frac{s\Delta_k^2}{8}}\\
        &\leq \frac{8}{\Delta_k^2}e^{-\frac{(\zeta(t) - 1)\Delta_k^2}{8}}\\
        &\leq \frac{8}{\Delta_k^2}e^{\frac{\Delta_k^2}{8}}e^{-\frac{\zeta(t)\Delta_k^2}{8}}\\
        &\leq \frac{8}{\Delta_k^2}e^{\frac{\Delta_k^2}{8}}e^{-2\ln (\frac{t}{2(1 + 2c)K})}\\
        &\leq \frac{64(1 + 2c)^2K^2}{(\Delta_kt)^2}
    \end{split}
\end{align}
where the penultimate inequality follows from the condition on $t$ in the lemma and the final inequality follows from the fact that $e^{\frac{x}{8}} < 2, \text{ for } 0 < x \leq 1$. A similar analysis can be done for the third term on the RHS of \eqref{eq:lem2}. The lemma then follows by combining \eqref{eq:lem2}, \eqref{eq:lem2a} and \eqref{eq:lem2b}.
\end{proof}
Now, we use Lemma \ref{lemma4} to upper bound the number of times a sub-optimal arm is pulled in the Boltzmann exploration phase of ABE (or the second inner loop of ABE), until the end of time horizon $T$.
\begin{corollary}\label{corr1}
Let $\tau''_k(T)$ denote the number of times a sub-optimal arm $k$ is pulled during the Boltzmann exploration phase of ABE (or the second inner loop of ABE), until the end of time horizon $T$. Then:
\begin{align*}
    \mathbb{E}[\tau''_k(T)] \leq O(\frac{c^2K^2}{\Delta_k^2}) + 2(1 + 2c) K 2^{\gamma_k},
\end{align*}
where $\gamma_k = \{\frac{16\ln 2}{ \Delta_k^2}\frac{2(\alpha + 1)}{c}\}^{\frac{1}{\alpha}}$, as defined in Lemma \ref{lemma4}.
\begin{proof}
By definition: 
\begin{align*}
    \mathbb{E}[\tau''_k(T)] = \sum_{t = 1}^T \mathbbm{1}\{t \in \text{2nd inner loop}\}\mathbb{E}[\mathbbm{1}\{k_t = k\}].
\end{align*}
Therefore, we have:
\begin{align*}
    \mathbb{E}[\tau_k''(T)] &\leq 2(1 + 2c)K2^{\gamma_k} + \sum_{t = 2(1 + 2c)K2^{\gamma_k} + 1}^Tp_k(t)\\
    &\leq 2(1 + 2c)K2^{\gamma_k} + O(\frac{c^2K^2}{\Delta_k^2}),
\end{align*}
where the last step follows from Lemma \ref{lemma4}.
\end{proof}
\end{corollary}
\begin{theorem}\label{theorem1}
The regret achieved by the ABE algorithm (Algorithm \ref{alg:abe}) for the stochastic multi-armed bandit problem is:
\begin{align*}
    \mathbb{E}[R(T)] &\leq \sum_{k\neq 1}\bigg\{c\frac{\Delta_k\{\log(\frac{T}{K} + 2) + 1\}^{\alpha + 1}}{\alpha + 1}\\&\hspace{5mm} + 2(1 + 2c) K 2^{\gamma_k}\Delta_k + O(\frac{c^2K^2}{\Delta_k})\bigg\}
\end{align*}
where $\gamma_k = \{\frac{16\ln 2}{ \Delta_k^2}\frac{2(\alpha + 1)}{c}\}^{\frac{1}{\alpha}}$, as defined in Lemma \ref{lemma4}.
\end{theorem}
\begin{proof}
Combining Lemma \ref{lemma3} and Corollary \ref{corr1}, we get the result.
\end{proof}
\begin{corollary}\label{corr2}
The regret achieved by the ABE algorithm for the stochastic multi-armed bandit problem is $O(K\log^{1 + \alpha} T)$. For $\alpha = 1$, the regret achieved by the ABE algorithm for the stochastic multi-armed bandit problem is $O(K\log^2 T)$. 
\end{corollary}
\begin{proof}
Follows from Theorem \ref{theorem1}.
\end{proof}
\subsection{Analysis of Algorithm 2}

The analysis of Algorithm 2 is similar to that of Algorithm \ref{alg:abe}. We only present the following theorem and its corollary for completeness, some additional details are provided in the supplementary material.
\begin{theorem}\label{theorem2}
The regret achieved by the ABE algorithm (Algorithm 2) for bandit problems with graph-structured feedback is:
\begin{align*}
    \mathbb{E}[R(T)] &\leq \sum_{k\neq 1, k \in \xi}\bigg\{c\frac{\Delta_k\{\log(\frac{T}{K} + 2) + 1\}^{\alpha + 1}}{\alpha + 1}\bigg\}\\&\hspace{5mm} + \sum_{k\neq 1}\bigg\{2(1 + 2c) K 2^{\gamma_k}\Delta_k + O(\frac{c^2K^2}{\Delta_k})\bigg\}
\end{align*}
where $\gamma_k = \{\frac{16\ln 2}{ \Delta_k^2}\frac{2(\alpha + 1)}{c}\}^{\frac{1}{\alpha}}$, as defined in Lemma \ref{lemma4}.
\end{theorem}
\begin{corollary}
The regret achieved by the ABE algorithm (Algorithm 2) for bandit problems with graph-structured feedback is $O(\beta(G)\log^{1 + \alpha}T)$, where $\alpha >0.$ 
\end{corollary}
\begin{proof}
Follows from Theorem \ref{theorem2} and the fact that $|\xi| \leq \beta(G)$.
\end{proof}

\section{MAB Experiments}
In this section we describe a real-world multi-armed bandit application (see \cite{Fekete10} and \cite{kegl2009boosting} for more details).
We then briefly describe the benchmark algorithms we compare against and the standard datasets we implement the algorithms on. 

\textbf{Bandit based Boosting:} Boosting is an ensemble learning method which produces a strong learner as a combination of weak base learners. The learning is based on the training data: $\{(x^{(1)}, y^{(1)}), \dots , (x^{(n)}, y^{(n)})\}$ where $x^{(i)}$ is the $i^{\text{th}}$ input vector and $y^{(i)} \in \{-1, +1\}^L$ is a label vector with $y_l = +1$ iff $l$ is the true label.
AdaBoost and its multiclass generalization AdaBoost.MH are among the most commonly used boosting algorithms introduced by \cite{Freund97} and \cite{Schapire99} respectively.

\cite{Fekete10} proposed a bandit based AdaBoost.MH algorithm, namely AdaBoost.MH.EXP3.P, to speed up the per-iteration complexity of AdaBoost.MH. This bandit based AdaBoost.MH works by dividing the set of weak classifiers into possibly overlapping subsets and selecting a subset at each iteration using a bandit algorithm (more details in the supplementary material). In the next section, we will compare our algorithm to the best algorithm (EXP3.P) in \cite{Fekete10}.

\textbf{Dataset Description:} We evaluated our algorithm on the bandit based boosting application, applying it to benchmark datasets used in \cite{Fekete10}. These datasets include MNIST, USPS, UCI letter, UCI pendigit, and UCI ISOLET which are all frequently used for evaluating classification algorithms. Statistics for these datasets are summarized in Table \ref{table:datasets}.

\begin{table}[ht]
\centering
\begin{tabular}{c c c c c}
\hline \hline
Dataset & \# Features & \# Classes & \# Train & \# Test \\ [0.5ex]
\hline
MNIST & 784 & 10 & 60,000 & 10,000 \\
USPS & 256 & 10 & 7,291 & 2,007 \\
letter & 16 & 26 & 16,000 & 4,000 \\ 
pendigit & 16 & 10 & 7,494 & 3,498 \\
ISOLET & 617 & 26 & 6,238 & 1,559 \\
\end{tabular}
\caption{Dataset Statistics}
\label{table:datasets}
\end{table}


\textbf{Benchmark Algorithms:} We compare our algorithm with two different bandit algorithms: EXP3.P and BGE. EXP3.P \cite{Auer95} is a popular adversarial bandit algorithm applied to the bandit based boosting application in \cite{Fekete10}. 
Boltzmann-Gumbel Exploration or BGE (see \cite{Gumbel}) is the second algorithm we compare our algorithm with. BGE samples an arm according to $\arg\max_i \{\hat{\mu}_i + C \sqrt{Z_{i, t} / N_{i, t}}\}$, where $Z_{i,t}$ is a standard Gumbel random variable and $N_{i,t}$ is the number of times arm $i$ has been sampled until time $t$. The authors suggest choosing $C$ to be an estimate of the subgaussian parameter of the rewards, so we choose $C = \frac{1}{4}$ as the rewards in this experiment lie between $0$ and $1$. 

\begin{figure*}[h]
\centering
\vskip -0.0in
    \begin{subfigure}[]{0.3\textwidth}
        \centering
        \includegraphics[scale = 0.35]{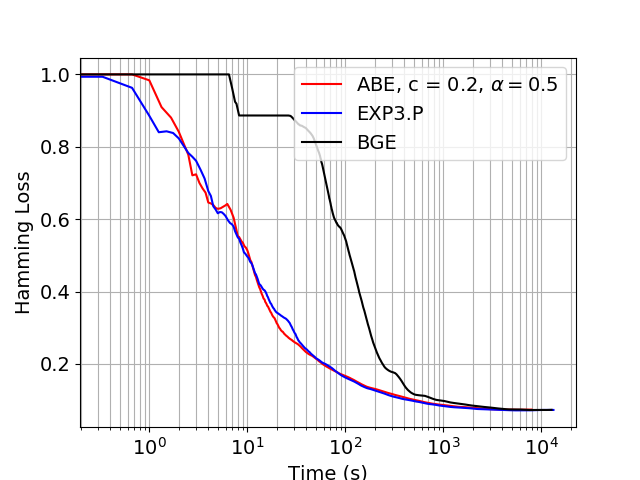}
        \caption{MNIST}
        \label{fig:MNIST_comparison}
    \end{subfigure} 
    \begin{subfigure}[]{0.3\textwidth}      
        \centering
        \includegraphics[scale = 0.35]{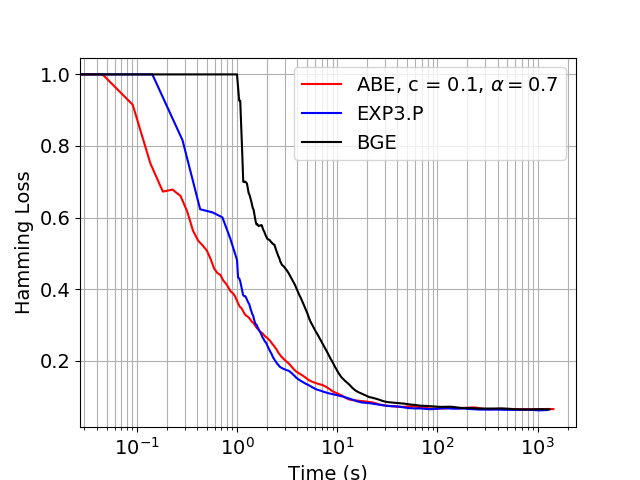}
        \caption{USPS}
        \label{fig:usps_comparison}
    \end{subfigure} 
    \begin{subfigure}[]{0.3\textwidth}
        \centering
        \includegraphics[scale = 0.35]{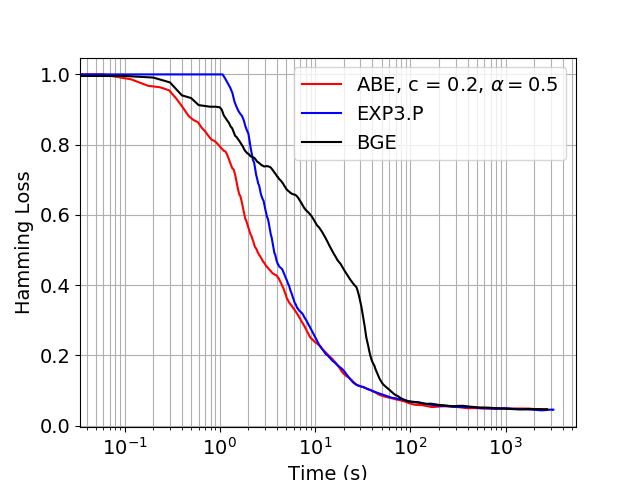}
        \caption{ISOLET}
        \label{fig:isolet_comparison}
    \end{subfigure} %
    \caption{Comparison of different bandit algorithms for boosting.}
    \label{fig:comparison_figure}
    \begin{subfigure}[]{0.3\textwidth}
        \centering
        \includegraphics[scale = 0.35]{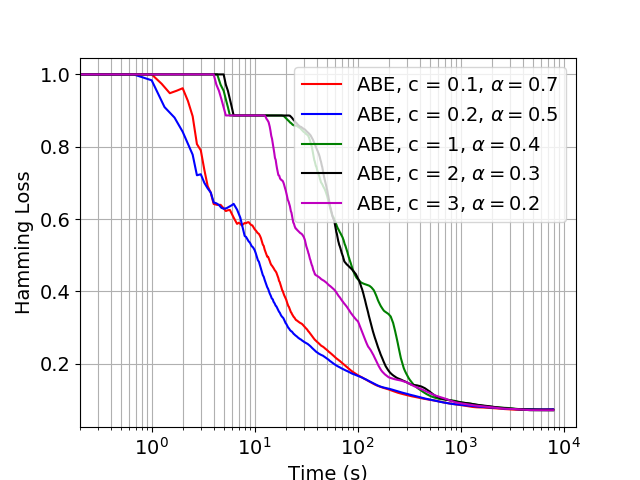}
        \caption{MNIST}
         \label{fig:mnist_varying_params}    
    \end{subfigure}
    \begin{subfigure}[]{0.3\textwidth}
        \centering
        \includegraphics[scale = 0.35]{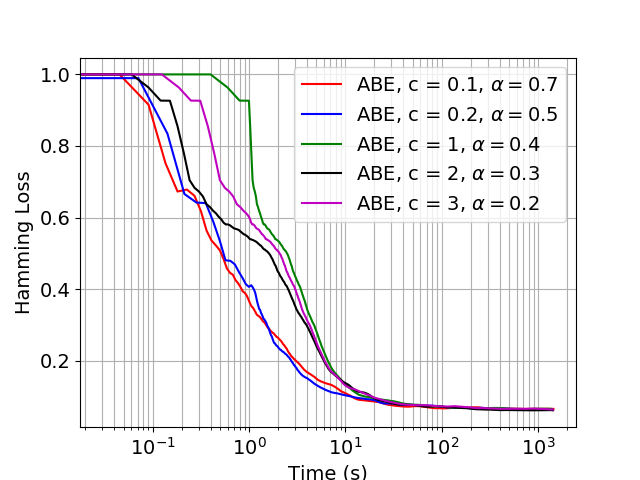}
        \caption{USPS}
        \label{fig:usps_varying_params}
    \end{subfigure} %
    \begin{subfigure}[]{0.3\textwidth}    
      \centering
      \includegraphics[scale=0.35]{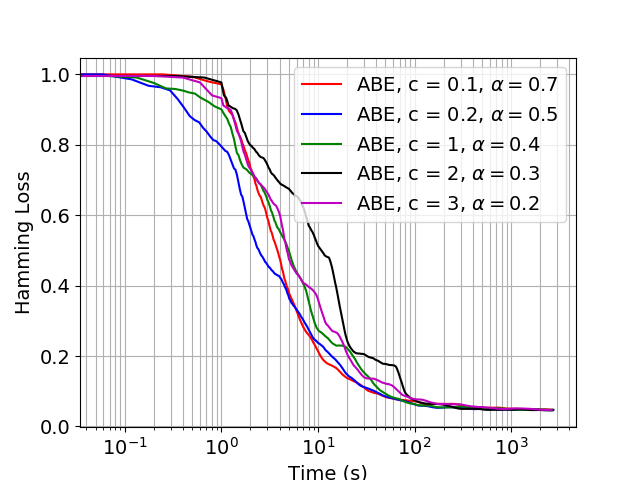}
      \caption{ISOLET}
        \label{fig:isolet_varying_params}   
    \end{subfigure} 
    \caption{Comparison of ABE with varying hyperparameters.}
\vskip -0.0in
\end{figure*} 

\begin{figure*}[h] 
\centering
\vskip 0.0in
\begin{subfigure}[]{\columnwidth}
        \centering
        \includegraphics[scale = 0.35]{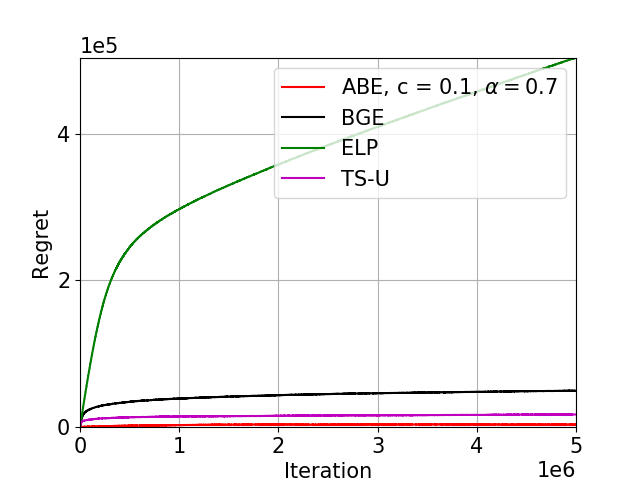}
        \caption{Epinions}
        \label{fig:epinions_comparison}
    \end{subfigure} %
    \begin{subfigure}[]{\columnwidth}
        \centering
        \includegraphics[scale=0.35]{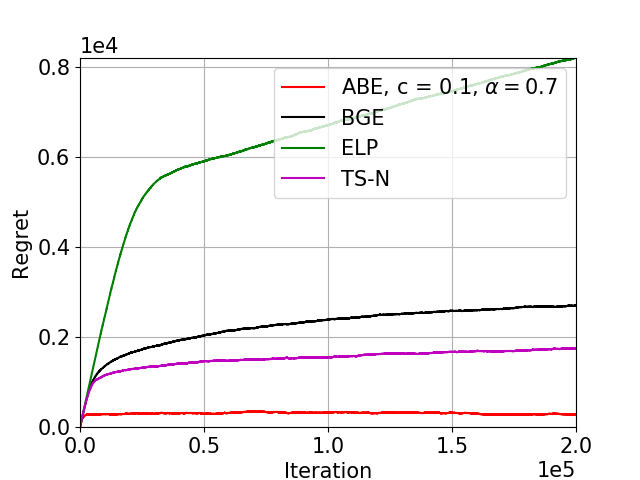}
        \caption{Facebook}
        \label{fig:facebook_comparison} 
    \end{subfigure} 
    \caption{Comparison of different bandit algorithms for graph-structured feedback problems.}
    \begin{subfigure}[]{\columnwidth}
        \centering
        \includegraphics[scale=0.35]{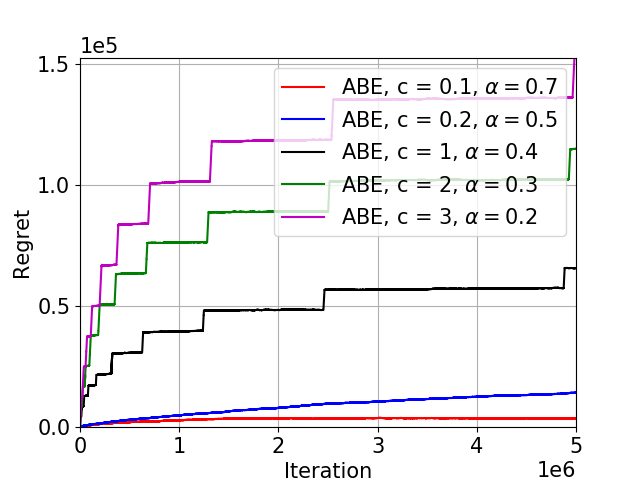}
        \caption{Epinions}
        \label{fig:epinions_varying_params}
    \end{subfigure} 
    \begin{subfigure}[]{\columnwidth}
        \centering
        \includegraphics[scale=0.35]{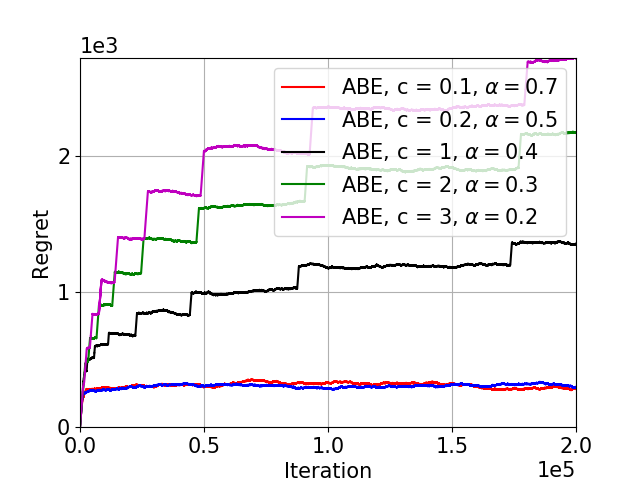}
        \caption{Facebook}
        \label{fig:facebook_varying_params}
    \end{subfigure} %
    \caption{Comparison of ABE with varying hyperparameters for graph-structured feedback problems.}
\vskip -0.0in
\end{figure*}


\textbf{Results:} We used the package described in \cite{MultiBoost} to run AdaBoost.MH.EXP3.P and added modules for our algorithm and BGE to compare the algorithms in a fair manner. For each dataset, all algorithms were run on the same computer and the testing error is plotted with respect to time (seconds). Because the initialization time for each algorithm may be different, we subtract this offset from the x-axis (time). Following \cite{Fekete10}, we averaged the test error at iteration $t$ over the last 20\% iterations. The results of the experiments comparing our algorithm with EXP3.P and BGE for MNIST, USPS and UCI ISOLET datasets are presented in Figures \ref{fig:MNIST_comparison}, \ref{fig:usps_comparison}, and \ref{fig:isolet_comparison} respectively. Due to space constraints, we do not include the figures for experiments on the UCI letter and UCI pendigit datasets. These figures can be found in the supplementary material. We note that our algorithm performs as well as or better than EXP3.P and BGE in all the experiments.

We used a training dataset to implement ABE with various hyperparameters $(c, \alpha)$ and used the largest $c$ which achieves the best performance in the test dataset (for which we report results). The best performing hyperparameters were observed to be $(c , \alpha) = (0.1, 0.7)$ for UCI pendigit, UCI letter and USPS, and $(0.2, 0.5)$ for UCI ISOLET and MNIST. The results of the experiments comparing the performance of our algorithm for different hyperparameter values are presented in Figures \ref{fig:mnist_varying_params},  \ref{fig:usps_varying_params}, and \ref{fig:isolet_varying_params}. Again, the corresponding figures for UCI letter and UCI pendigit datasets can be found in the supplementary material.

\section{Structured MAB Experiments}
In this section, we evaluate our algorithm (Algorithm 2) for bandit problems with graph-structured feedback on real-world graphs. 

\textbf{Recommendation in Social Networks:} We follow the experimental setup in \cite{Caron12} and \cite{Tossou17}. We consider a social network in which, at each iteration, a user is selected and an item is shown to the user who then rates the item. The item is also posted on the user's wall so that the friends of the user can view it, and it is assumed that they also rate this item. The goal is to show items to users who are likely to positively rate the shown items.

This problem can be formulated as a bandit problem with graph-structured feedback as follows. Users correspond to arms, and a friendship between two users defines an edge between the users. At each iteration the bandit algorithm recommends an item to some user $u$ and a feedback of $1$ (a positive rating) or $0$ (a negative rating) is received. Further, the friends of user $u$ also reveal their feedback to the bandit algorithm, serving as side information.

\textbf{Datasets:} To model graph-structured feedback, we follow the approach in \cite{Caron12} which appears to be the only major dataset evaluation of algorithms for MAB problems with such a feedback structure.
However, the Flixster dataset used in \cite{Caron12} is no longer available. So we used the Epinions dataset and preprocessed it in a similar manner. The details regarding the preprocessing steps are provided in the supplementary material.

The second dataset we use is the Facebook graph dataset from the Stanford Network Analysis Project (SNAP) \cite{Facebook}. Among the $10$ networks in the dataset, we used the network named 0.edges which has a list of pairs specifying an undirected edge between the pair. There were $333$ distinct nodes in the file and $5038$  edges.

We generate the rewards as follows. As the graph is created, a special node is sampled uniformly at random. This special node has a Bernoulli reward of mean $0.75$ while all other nodes produce Bernoulli rewards with mean $0.5$. This reward structure is motivated by the example in \cite{ELP} used to prove a lower bound on adversarial bandits on graphs. Edges determine the side information, i.e., if there is an edge between nodes $u$ and $v$, the reward for $v$ is also observed when sampling $u$ and vice versa.

\textbf{Benchmark Algorithms:} We compare our algorithm to the following algorithms: (i) ELP \cite{ELP}, which was designed for adversarial bandits with graph-structured feedback. ELP can be used with both undirected and directed graph-structured feedback. (ii) 
TS-U and TS-N, which are based on Thompson sampling, which were designed for directed and undirected graphs respectively \cite{TSN}. It is worth noting that ELP requires some knowledge of the graph structure, while our algorithm, TS-N and TS-U do not.

\textbf{Results:} For the Epinions recommendation experiment, ABE, TSU, and Boltzmann-Gumbel all demonstrate good performance in terms of regret as illustrated in figure \ref{fig:epinions_comparison}. ABE outperforms the other algorithms. Again, we observed that small hyperparameter values ($c = 0.1, \alpha = 0.7$) performed the best, as illustrated in Figure \ref{fig:epinions_varying_params}.


For the Facebook experiment, we plot the regret for the simulated data averaged over $10$ realizations and compare our algorithm with BGE, ELP, and TS-N as shown in Figure \ref{fig:facebook_comparison}. We first implemented ABE with different hyperparameter values, as shown in Figure \ref{fig:facebook_varying_params}. Then we used the hyperparameter values $c = 0.1, \alpha = 0.7$ (which performed the best) to compare ABE with other algorithms, as shown in Figure \ref{fig:facebook_comparison}. The parameters for all other algorithms were chosen to be the parameters theoretically justified in each paper. Because ELP has a step size that depends on the independence number of the graph, we tested the algorithm with different independence number estimates and picked the one that performed the best.

We see that our algorithm performs significantly better than all the other algorithms.

\section{Conclusion}
In this paper, we propose a simple modification to the popular Boltzmann exploration technique, which achieves near-logarithmic regret in the stochastic MAB setting. Additionally, the performance of our algorithm on real datasets suggests that adding a small amount of pure exploration to Boltzmann exploration leads to excellent performance in various bandit applications.
\section*{Acknowledgement}

We would like to thank St\'ephane Caron and Branislav Kveton for providing us with the preprocessing code for the Epinions dataset.

\bibliography{example_paper}
\bibliographystyle{icml2019}
\newpage
\section*{Supplementary Material}
\setcounter{section}{0}
\section{Pseudocode for Algorithm 2}
In this section, we present the pseudocode for Algorithm \ref{alg:abeg} as mentioned in the main body of the paper.
\setcounter{algorithm}{1}
\begin{algorithm}[H]
\caption{ABE for graph-structured bandits}\label{alg:abeg}
\begin{algorithmic}[1]\\
\textbf{parameters}: $\eta , \alpha , c > 0$.\\
\textbf{initialize}: $r_i(0) = 0$, $\tau_i(0) = 0 $, $\forall i \in [K]$, and $t = 0$.
\For{$i = 1, 2, ...,$}
\For{$j = 1, 2, 3, ..., |\xi|\times\lfloor ci^{\alpha}\rfloor$}
\State $t = t + 1$.
\State Play arm $k_t = \xi_j$.
\State Receive reward $x_{k_t}(t)$.
\State Observe rewards $x_{n_t}(t), \forall n_t \in N_{k_t}$.
\State $\tau_{m_t}(t) = \tau_{m_t}(t - 1) + 1, \forall m_t \in k_t\cup N_{k_t}$.
\State $r_{m_t}(t) = \frac{(\tau_{m_t}(t) - 1)r_{m_t}(t - 1) + x_{m_t}(t)}{\tau_{m_t}(t)}, \forall m_t \in k_t\cup N_{k_t}$.
\EndFor
\For{$j = 1, 2, 3, ..., K\times2^{i}$}
\State $t = t + 1$.
\State Compute $p_k(t) \propto \exp\big(\eta \sqrt{(t-1)}r_k\big)$, $\forall k $.
\State Play arm $k_t \sim p_k(t)$.
\State Receive reward $x_{k_t}(t)$.
\State Observe rewards $x_{n_t}(t), \forall n_t \in N_{k_t}$.
\State $\tau_{m_t}(t) = \tau_{m_t}(t - 1) + 1, \forall m_t \in k_t\cup N_{k_t}$.
\State $r_{m_t}(t) = \frac{(\tau_{m_t}(t) - 1)r_{m_t}(t - 1) + x_{m_t}(t)}{\tau_{m_t}(t)}, \forall m_t \in k_t\cup N_{k_t}$.
\EndFor
\EndFor\label{WISHForLoop}
\end{algorithmic}
\end{algorithm}
\section{Proofs}
\subsection{Proof for Lemma 1}
Until the end of the outer iteration $i$, the total number of time slots elapsed is $M_i = \sum_{k = 1}^i\{(2^k)K + \lfloor ck^{\alpha}\rfloor K\}$. Therefore:
\begin{align*}
    &\sum_{k = 1}^{i_t - 1}\{(2^k)K + \lfloor ck^{\alpha}\rfloor K\} \leq t \leq     \sum_{k = 1}^{i_t}\{(2^k)K + \lfloor ck^{\alpha}\rfloor K\}\\
    &\Rightarrow \sum_{k = 1}^{i_t - 1}2^k \leq \frac{t}{K} \leq \sum_{k = 1}^{i_t}\{2^k + ck^{\alpha}\}\\
    &\Rightarrow 2(2^{i_t - 1} - 1) \leq \frac{t}{K} \leq 2(2^{i_t} - 1) + \sum_{k = 1}^{i_t}ck^{\alpha}\\
    &\Rightarrow 2(2^{i_t - 1} - 1) \leq \frac{t}{K} \leq 2(2^{i_t} - 1) + \int_{0}^{i_t + 1}ck^\alpha dk\\
    &\Rightarrow 2(2^{i_t - 1} - 1) \leq \frac{t}{K} \leq 2(2^{i_t} - 1) + c\frac{(i_t + 1)^{\alpha + 1}}{\alpha + 1}\\
    &\Rightarrow 2^{i_t} - 2 \leq \frac{t}{K} \leq 2(2^{i_t} - 1) + c(2^{i_t + 1} + 1)\\        &\Rightarrow 2^{i_t} - 2 \leq \frac{t}{K} \leq 2(1 + 2c)2^{i_t}
\end{align*}
where the penultimate inequality follows from the fact that $\frac{(i + 1)^{\alpha + 1}}{\alpha + 1} \leq (i + 1)^2 \leq 2^{i + 1} + 1$, for $0 < \alpha \leq 1$. The lemma follows from the last inequality.
\subsection{Analysis of Algorithm 2}

The analysis for Algorithm \ref{alg:abeg} is very similar to the analysis of Algorithm 1. In fact, to analyze Algorithm \ref{alg:abeg} we will use multiple results obtained in Section 3.1 of the paper. It can be easily shown that Lemma 1 holds for Algorithm \ref{alg:abeg}. Next, we have the following lemma which is similar to Lemma 2 for Algorithm 1 from Section 3.1.
\setcounter{lemma}{4}
\begin{lemma}\label{lemma5}
The number of times any arm $k$ has been observed until the beginning of time $t$, for any $t$ such that $\log(\frac{t}{2(1 + 2c)K}) \geq \{\frac{2(\alpha + 1)}{c}\}^{\frac{1}{\alpha}}$ and $t$ falls in the Boltzmann exploration phase of ABE for graph-structured bandits (Algorithm \ref{alg:abeg}), is lower bounded by:
\begin{align*}
    \tau_k(t) \geq \frac{c}{2}\frac{\{\log(\frac{t}{2(1 + 2c)K})\}^{\alpha + 1}}{\alpha + 1}, \forall k \in [K].
\end{align*}
\end{lemma}
\begin{proof}
Similar to Lemma 2.
\end{proof}
Next, we upper bound the number of times each arm will be pulled in the pure exploration phase of Algorithm \ref{alg:abeg}, until the end of time horizon $T$.
\begin{lemma}\label{lemma6}
Let $\tau'_k(T)$ denote the number of times arm $k$ is pulled during the pure exploration phase of ABE for graph-structured bandits (Algorithm \ref{alg:abeg}), until the end of time horizon $T$. Then:
\begin{align*}
    \tau'_k(T) &\leq c\frac{\{\log(\frac{T}{K} + 2) + 1\}^{\alpha + 1}}{\alpha + 1}, \forall k \in \xi,\\
    \tau'_k(T) &= 0, \forall k \notin \xi.
\end{align*}
\end{lemma}
\begin{proof}
Similar to Lemma 3.
\end{proof}
Next, we observe that Lemma 4 and Corollary 1 from Section 3.1 also hold for Algorithm \ref{alg:abeg}. Therefore, as in Section 3.1, combining Lemma \ref{lemma6} and Corollary 1 (for Algorithm \ref{alg:abeg}), we get Theorem 2.
\section{MAB Experiments}
\subsection{Additional Details for Bandit based Boosting}
Boosting is an ensemble learning method which produces a strong learner as a combination of weak base learners. The learning is based on the training data: $\{(x^{(1)}, y^{(1)}), \dots , (x^{(n)}, y^{(n)})\}$ where $x^{(i)}$ is the $i^{\text{th}}$ input vector and $y^{(i)} \in \{-1, +1\}^L$ is a label vector with $y_l = +1$ iff $l$ is the true label.
AdaBoost and its multiclass generalization AdaBoost.MH are among the most commonly used boosting algorithms introduced by \cite{Freund97} and \cite{Schapire99} respectively. These algorithms work by iterating through a training dataset and updating weights $w_{i,l}$ used to produce base learners at each iteration, ultimately returning a strong learner as a linear combination of base learners. 

We focus on discrete AdaBoost.MH with $L$ labels, where base learners are represented by
\begin{equation}\label{eq:base}
\bm{h}(x) = \alpha \bm{v} \varphi (x)
\end{equation}
with a vote vector $\bm{v} \in \{-1, +1\}^L$, base coefficient $\alpha$ (specified later), and scalar base learner $\varphi$ chosen to maximize the \textit{edge}
\begin{equation}\label{eq:edge}
\gamma = \sum_{i=1}^n \sum_{l=1}^L w_{i,l} v_l \varphi (x^{(i)}) y_l^{(i)}.
\end{equation}
The corresponding base coefficient is given by $\alpha = \frac{1}{2} \log \frac{1 + \gamma}{1 - \gamma}$.

For a multiclass classification task, scalar base learners in AdaBoost.MH can be any classifier including neural networks and decision trees. A commonly used scalar base learner is a single stump learner which is a decision tree of depth $1$ of the form
\begin{equation}
\varphi_{j,b} (x) = \mathbbm{1}\{x_j \geq b\}
\end{equation}
where $x_j$ denotes the $j^{\text{th}}$ dimension of input $x$.

One significance of AdaBoost.MH is that it can \textit{provably} obtain zero training error after a sufficient number of iterations with adequate weak learners. This is achieved by computing the base learner $\bm{h}$ as in \eqref{eq:base} at each iteration such that \eqref{eq:edge} is maximized, and recursively updating weights according to the following equation
\begin{equation}\label{eq:weights}
w_{i,l} \propto w_{i,l} \exp\Big(-h_l (x^{(i)}) y_l^{(i)}\Big).
\end{equation}
It is shown that after $T$ iterations, the strong classifier $\bm{f}^{(T)} = \sum_{t=1}^T \bm{h}^{(t)}$ minimizes the exponential margin loss
\begin{align*}
R(f^{(T)}) & = \frac{1}{2n}\sum_{i = 1}^n \sum_{l = 1}^L\Big( \exp\big(- f_{l}^{(T)} (x^{(i)}) \big)\mathbbm{1}\{y_l^{(i)} = +1\} \\
& \quad + \frac{1}{L - 1} \exp\big( - f_l^{(T)}(x^{(i)}) y_l^{(i)}\big)\mathbbm{1}\{y_l^{(i)} = -1\}\Big)
\end{align*}
which upper bounds the Hamming loss.

\cite{Fekete10} proposed a bandit based AdaBoost.MH algorithm, namely AdaBoost.MH.EXP3.P, to speed up the per-iteration complexity. 
This bandit based AdaBoost.MH works by dividing the set of weak classifiers ${\cal H}$ into $K$ (possibly overlapping) subsets $\{{\cal H}_1 , \dots , {\cal H}_K\}$ and selecting an index $i \in \{1, \dots , K\}$ using a bandit algorithm. 
The maximization in \eqref{eq:edge} is then computed over the subset ${\cal H}_i$, significantly reducing the per-iteration time complexity from $\abs{{\cal H}}$ to $\abs{{\cal H}_i}$. The bandit algorithm then receives a reward of $-\log \sqrt{1 - \gamma^{(t)}}$ clipped to $1$ so that the reward is bounded between $[0,1]$. 
The authors show that with sufficiently expressive scalar base learners, the training error goes to zero when using an adversarial bandit algorithm. In the next section, we will compare our algorithm to the best algorithm (EXP3.P) in \cite{Fekete10}.
\subsection{Additional Results}
In this subsection, we present the results for UCI pendigit and UCI letter datasets. These were omitted from the main body of the paper due to space constraints. Figures \ref{fig:pendigit_comparison} and \ref{fig:letter_comparison} compare the performance of various algorithms on the UCI pendigit dataset and UCI letter dataset respectively. Figures \ref{fig:pendigit_varying_params} and \ref{fig:letter_varying_params} compare the performance of ABE with varying hyperparameters for the UCI pendigit dataset and UCI letter dataset respectively.

\begin{figure*}[h] 
\centering
\vskip 0.0in
\begin{subfigure}[]{\columnwidth}
        \centering
        \includegraphics[scale = 0.35]{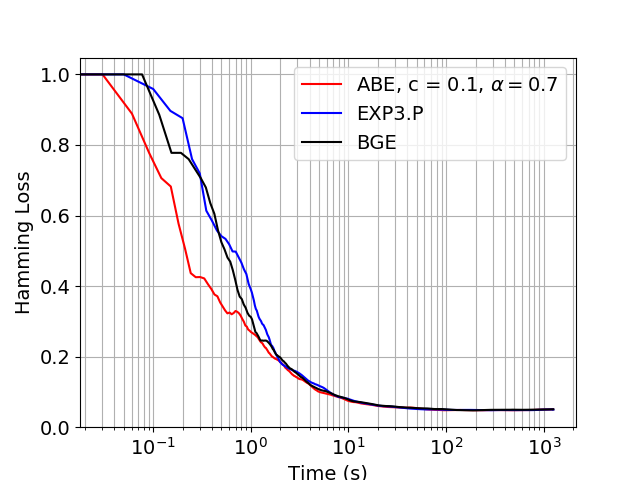}
        \caption{Pendigit}
        \label{fig:pendigit_comparison}
    \end{subfigure} %
    \begin{subfigure}[]{\columnwidth}
        \centering
        \includegraphics[scale=0.35]{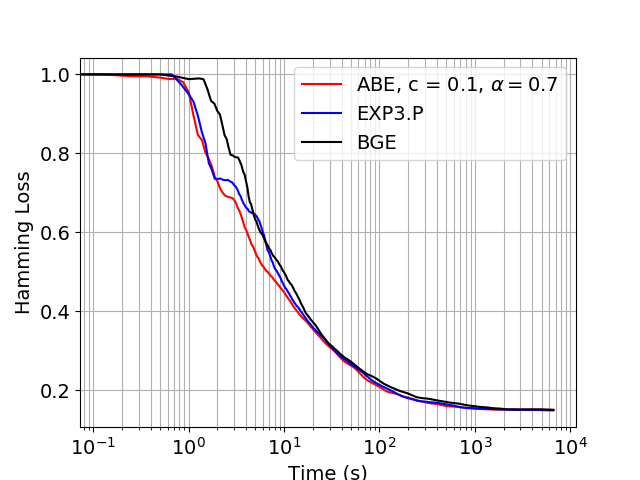}
        \caption{Letter}
        \label{fig:letter_comparison} 
    \end{subfigure} 
    \caption{Comparison of different bandit algorithms for boosting.}
    \begin{subfigure}[]{\columnwidth}
        \centering
        \includegraphics[scale=0.35]{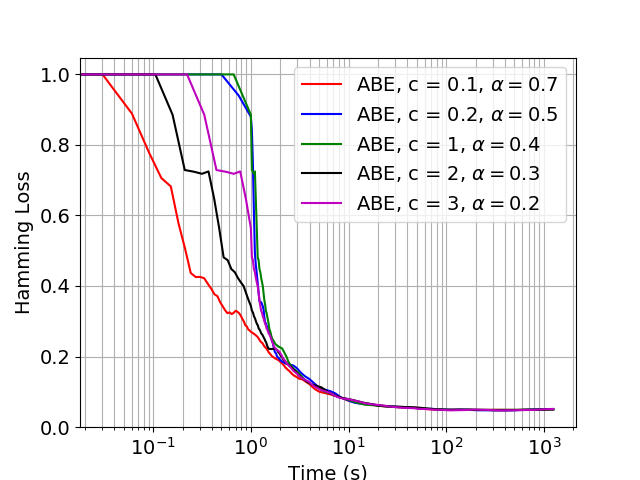}
        \caption{Pendigit}
        \label{fig:pendigit_varying_params}
    \end{subfigure} 
    \begin{subfigure}[]{\columnwidth}
        \centering
        \includegraphics[scale=0.35]{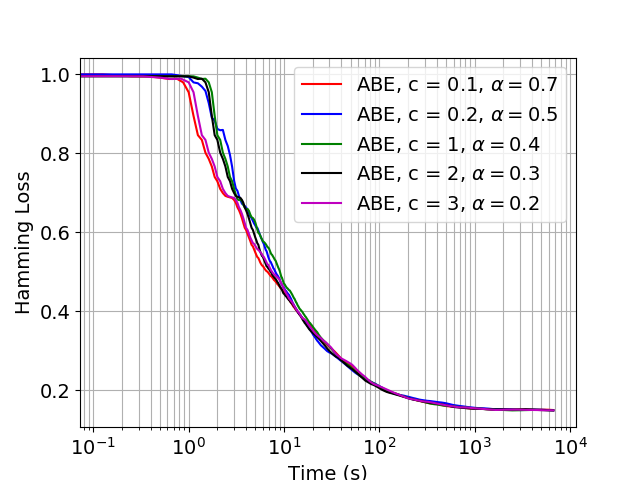}
        \caption{Letter}
        \label{fig:letter_varying_params}
    \end{subfigure} %
    \caption{Comparison of ABE with varying hyperparameters.}
\vskip -0.0in
\end{figure*}

\section{Structured MAB Experiments}
\subsection{Additional Details about Preprocessing for the Epinions Experiment}
As mentioned in the main body of the paper, to model graph-structured feedback, we follow the approach in \cite{Caron12} which appears to be the only major dataset evaluation of algorithms for MAB problems with such a feedback structure.
However, the Flixster dataset used in \cite{Caron12} is no longer available. So we used the Epinions dataset instead. The dataset was crawled in \cite{Epinions} and is publicly available at \url{https://www.cse.msu.edu/~tangjili/datasetcode/}. The dataset has been used in prior literature to test recommendation algorithms \cite{Jamali10, Ghenai16}. 

The dataset consists of ``trust relations" between two users and user-product ratings, where each user rates a product from $1$ to $5$. A trust relation defines a directed edge between two users. The following data preprocessing steps are similar to those in \cite{Caron12}: to estimate ratings that are not included in the dataset, we use a matrix completion algorithm \cite{Jamali10}. Given the complete user-pair rating matrix, we replace the ratings by $1$ if the ratings are above the average over all user-product ratings, and $0$ otherwise. Users who rate only a few products and products rated by only a few users generally provide noisy ratings. Therefore, we eliminate users who rate fewer than $30$ products and products rated by fewer than $30$ users. This results in $4,711$ users, $48,350$ edges, and $3,530$ products. A majority of the preprocessing code was provided by the authors of \cite{Caron12}.

We randomly partition the user-product matrix such that $6 / 7$ of the products represented in the user-product ratings matrix are used for validating hyperparameters for our algorithm and the remaining $1 / 7$ user-product pairs are used for testing all algorithms. This ratio of training and test data was motivated by the publicly available training and test MNIST datasets.

\end{document}